\documentclass[final,12pt]{arxiv} 

\title[Finite Samples for Score Matching]{Finite Sample Bounds for Learning with Score Matching}
\usepackage{times}

\coltauthor{%
\Name{Devin Smedira}
\Email{smedira@mit.edu}\\
\addr{Operations Research Center, Massachusetts Institute of Technology}
\AND
 \Name{Abhijith Jayakumar} \Email{abhijithj@lanl.gov}\\
 \addr Theoretical division, Los Alamos National Lab
 \AND
 \Name{Sidhant Misra} \Email{sidhant@lanl.gov}\\
 \addr Theoretical division, Los Alamos National Lab
 \AND
 \Name{Marc Vuffray} \Email{vuffray@lanl.gov}\\
 \addr Theoretical division, Los Alamos National Lab%
 \AND
 \Name{Andrey Y. Lokhov} \Email{lokhov@lanl.gov}\\
 \addr Theoretical division, Los Alamos National Lab%
}

\usepackage{subfiles}
\usepackage{algorithm}
\usepackage{algorithmic}

\DeclareMathOperator*{\otimesL}{\otimes}

\allowdisplaybreaks
\newcounter{condition}
\newenvironment{condition}[1][]{\refstepcounter{condition}\par\medskip
   \noindent \textbf{Condition~\thecondition. #1} \rmfamily}{\medskip}

\begin{document}


\newcommand{\TODO}[1]{{\color{red} #1}}

\newcommand{\R}{{\mathbb{R}}}
\newcommand{\Z}{{\mathbb Z}}
\newcommand{\Q}{{\mathbb Q}}
\newcommand{\C}{{\mathbb C}}
\newcommand{\N}{{\mathbb N}}
\newcommand{\E}{\mathbb{E}}
\newcommand{\Var}{\rm Var}

\renewcommand{\P}{\mathcal{P}}
\newcommand{\D}{\mathcal{D}}
\renewcommand{\L}{\mathcal{L}}
\newcommand{\Ee}{\tilde{\E}}

\newcommand{\la}{\leftarrow}
\newcommand{\ra}{\rightarrow}

\newcommand{\<}{\langle}
\renewcommand{\>}{\rangle}

\renewcommand{\c}{\textbf{c}}             
\renewcommand{\b}{\textbf{b}}               
\renewcommand{\a}{\textbf{a}}               
\renewcommand{\d}{\textbf{d}}                    
\renewcommand{\u}{\textbf{u}}                   
\newcommand{\m}{\textbf{m}}  
\newcommand{\n}{\textbf{n}}    
\newcommand{\0}{\textbf{0}}                 
\newcommand{\1}{\textbf{1}}        
\newcommand{\s}{\textbf{s}}        
\newcommand{\p}{\textbf{p}}        
\newcommand{\q}{\textbf{q}}        
\newcommand{\f}{\textbf{f}}        
\newcommand{\g}{\textbf{g}}   
\renewcommand{\k}{\textbf{k}}             
\newcommand{\h}{\textbf{h}}   
\newcommand{\w}{\textbf{w}}         
\newcommand{\x}{\textbf{x}}       
\newcommand{\y}{\textbf{y}}       
\newcommand{\z}{\textbf{z}}        
\renewcommand{\r}{\textbf{r}} 
\renewcommand{\t}{\textbf{t}} 
\renewcommand{\v}{\textbf{v}} 
\renewcommand{\j}{\textbf{j}} 
\newcommand{\A}{\mathcal{A}}
\newcommand{\K}{\mathcal{K}}
\newcommand{\M}{\mathcal{M}}
\renewcommand{\S}{\mathcal{S}}
\renewcommand{\h}{\mathcal{h}}
\renewcommand{\g}{\mathcal{g}}

\renewcommand{\(}{\left (}
\renewcommand{\)}{\right )} 

\newcommand{\OPT}{\operatorname{OPT}}
\newcommand{\vol}{\operatorname{Vol}}
\newcommand{\Tr}{\operatorname{Tr}}
\newcommand{\Supp}{\operatorname{Supp}}

\newcounter{eqs}
\setcounter{eqs}{0}
\newcommand{\eql}[1]{\label{eq:#1} \tag{\arabic{eqs}} \stepcounter{eqs}}

\newcommand\eqa{\stackrel{\mathclap{\normalfont\mbox{a}}}{=}}
\newcommand\eqb{\stackrel{\mathclap{\normalfont\mbox{b}}}{=}}
\newcommand\eqc{\stackrel{\mathclap{\normalfont\mbox{c}}}{=}}
\newcommand\eqd{\stackrel{\mathclap{\normalfont\mbox{d}}}{=}}
\newcommand\eqe{\stackrel{\mathclap{\normalfont\mbox{e}}}{=}}

\newcommand\leqa{\stackrel{\mathclap{\normalfont\mbox{a}}}{\leq}}
\newcommand\leqb{\stackrel{\mathclap{\normalfont\mbox{b}}}{\leq}}
\newcommand\leqc{\stackrel{\mathclap{\normalfont\mbox{c}}}{\leq}}
\newcommand\leqd{\stackrel{\mathclap{\normalfont\mbox{d}}}{\leq}}
\newcommand\leqe{\stackrel{\mathclap{\normalfont\mbox{e}}}{\leq}}

\newtheorem{fact}{Fact}

\newtheorem{assumption}{Assumption}

\maketitle

\begin{abstract}%
Learning of continuous exponential family distributions with unbounded support remains an important area of research for both theory and applications in high-dimensional statistics. In recent years, score matching has become a widely used method for learning exponential families with continuous variables due to its computational ease when compared against maximum likelihood estimation. However, theoretical understanding of the statistical properties of score matching is still lacking. In this work, we provide a non-asymptotic sample complexity analysis for learning the structure of exponential families of polynomials with score matching. The derived sample bounds show a polynomial dependence on the model dimension. These bounds are the first of its kind, as all prior work has shown only asymptotic bounds on the sample complexity. 
\end{abstract}


\section{Introduction}

Learning of many-variable distributions is a fundamental primitive in machine learning. Often, it is natural to parametrize a target distribution via a class of energy functions $E_\theta$ as
\[p_\theta(x) = \frac{\exp(-E_\theta(x))}{Z_\theta} \]
up to a normalizing constant $Z_\theta$, reducing the problem of learning the distribution to learning the parameter $\theta$ of the energy function. Though statistically efficient, the standard Maximum-Likelihood estimation method applied to this problem suffers from a computational bottleneck by which the learner is forced to implicitly or explicitly compute the partition functions $Z_\theta$ of models in the hypothesis class \citep{montanari_computational_2015}. 

In recent years, score matching \citep{hyvarinen2005estimation} has emerged as a popular method to tackle such learning tasks, primarily owing to it's computational tractability. Score matching works directly with the \emph{score} of the distribution ($\nabla_x \log(p_\theta(x))$), minimizing the expected difference of this quantity with the score of a parametrized hypothesis. While initially seeming impossible to compute without prior knowledge of the true model, this loss can be rewritten via integration by parts as something easily computable from samples. This technique has several benefits, first among them being that the reliance on derivatives of the Gibbs distribution avoids the difficulty of computing a partition function.

In this work, we study distributions parameterized by polynomial energy functions with learnable coefficients, called \textit{exponential families}. Exponential families are a common family of parametric distributions which have been studied since the works of \cite{darmois1935lois}, \cite{koopman1936distributions}, and \cite{pitman1936sufficient}. Importantly for our designs, the score-matching loss on exponential families can be reduced to a fairly benign quadratic optimization problem over the polynomial coefficients. Our primary contribution is to provide a finite sample bound to learn the \textit{structure} of exponential families, meaning the coefficients for the maximal monomials of the energy function, with score-matching. Learning such structural results expose full conditional independence and Markov property of the model.
Additionally, we will show a finite sample bound for the total parameter recovery of models parameterized by multi-linear polynomials. In both cases, these bounds will be polynomial in the dimension of the model.

While many prior works have looked at the statistical properties of score matching, a few are highly relevant to this study of exponential families. \cite{sriperumbudur2017density} derives a model agnostic asymptotic sample convergence rate for the score matching estimator on exponential families. A more refined analysis of the model dependent properties of score matching is provided by \cite{koehler2022statistical} for the case of distribution families with good isoperimetric properties. These techniques are applied to exponential families in \cite{pabbaraju2023provable} to derive an asymptotic sample bound which fully captures the dependence on the target family's dimension. In contrast, the aim of our present work is to prove first of their kind non-asymptotic sample complexity bounds for parameter recovery with score matching under natural model assumptions. 

Unlike previous works on score-matching which revolve around the study of the restricted Poincaré constant to relate score-matching to the maximum-likelihood estimator, our results are derived from the explicit study of the curvature of the score-matching loss.
We devise a method to relate local properties of the distribution to loss for the entire distribution. We expect this method to be easily generalized for other models with continuous energy models.

\subsection{Further related work}
\paragraph{Approximations to Maximum Likelihood}
A classical alternative to exact maximum likelihood for energy-based models is to replace the intractable log-partition gradient by a Markov chain estimate, as done in \emph{contrastive divergence} (CD-$k$), introduced in the context of training restricted Boltzmann machines \citep{hinton2002training}. Another broad family of approximations replaces the exact likelihood by variational bounds or mean-field approximations that trade statistical efficiency for tractability \citep{wainwright2008graphical}. In contrast, score matching sidesteps the partition function entirely by moving from likelihood to a score-based objective.

\paragraph{Discrete alphabets and graphical model learning}
For discrete distributions (e.g., Ising/Potts models), there is a large literature on algorithms that learn the energy  that avoid the partition function. Prominent examples include  logistic regression \citep{ravikumar2010high, wu2019sparse}, interaction screening \citep{vuffray2016interaction, vuffray2022efficient}, and Sparsitron \citep{Klivans2017}. In addition, ratio matching \citep{hyvarinen2007some} can be viewed as a direct discrete analogue of score matching that yields a tractable objective while retaining consistency under suitable conditions. However, the theory of ratio matching is not as well developed as the aforementioned examples. Conversely, several works attempted to generalize the estimators for discrete models to exponential families with continuous variables. \citep{shah2021computationally, shah2021learning} applied the interaction screening technique to the setting of continuous exponential families with pairwise interactions, however restricted to distributions with bounded support, and showed an exponential dependence of sample complexity on the domain bound. \citep{ren2021learning} introduced a modification of the interaction screening estimator that works for exponential families with unbounded support and numerically compared it to the pseudolikelihood estimator, but did not provide the statistical complexity analysis of the estimator.

\paragraph{Score matching and generative modeling}
Score matching also plays a central role in modern diffusion/score-based generative modeling, where one learns (variants of) the score of noise-perturbed data distributions and then samples by simulating a reverse-time SDE \citep{song2019generative, song2021scorebased}. There is a growing theory literature that studies when approximate scores yield guarantees on the distribution produced by the discretized reverse process
\citep{azangulov2024convergence,oko2023diffusion, tang2024adaptivity, yakovlev2025generalization, yakovlev2025implicit}.
\cite{lee2023convergence},  proves polynomial-time convergence guarantees for denoising diffusion / score-based generative modeling under very general assumptions on the data distribution, assuming $L^2$-accurate score estimates. Related convergence analyses  include Wasserstein guarantees for broader classes of score-based models \citep{gao2023wasserstein} and more recent results covering weaker smoothness regimes \citep{bruno2025wasserstein}.

\section{Problem Formulation and Main Results} \label{definitions}

In this section, we will formulate our model selection problem for exponential families as well as introduce the score-matching loss which underlies our technique.

\subsection{Exponential Families}
A generic exponential family on $\R^n$ has the form $p_\theta(x) \propto h(x) \exp(\<\theta, T(x)\>),$ where $h(x)$ is some base measure, $\theta$ is a vector of parameters, and $T(x)$ is a vector of basis functions. In this work, we will limit ourselves to exponential families with the basis functions in $T(x)$ consisting only of monomials in $x_1, \ldots x_n$ with degree at most $d$. 

We will use $\K$ to denote the \textit{factor set} which indexes the set of the basis functions (as well as the parameters) for a given exponential family. For each $k \in \K$ we will let $k_i$ denote the degree of $x_i$ in the corresponding basis function $f_k(x) = \prod_{i=1}^n x_i^{k_i}$. We will use $\partial k = \{i| k_i > 0\}$ to denote the support of $k$, $w = \max_{k \in \K} |\partial k|$ to denote the \textit{interaction order} of the model family and let $\K_i = \{k \in \K| i \in \partial k\}$. Throughout this work, we will use $\theta^*$ to denote the parameters of the true model we sample from. Further, we will assume there is some integer $\ell_1$-bound $B$ on the parameters associated with each variable, meaning $\sum\limits_{k \in \K_i} |\theta_k^*| \leq B.$ Finally, our analysis will require bounds on the tail decay of the true model, so we will limit ourselves to distributions $p_{\theta^*}$ satisfying the following condition.

\begin{condition}\label{Bound_Assum}
    There exists some constant $k > 0$ and integer $C_t$ satisfying $\max((\ln(2)/k)^{1/(d-1)},1) \leq C_t \leq e^n$ such that for any $s \geq C_t$, $\Pr(\|x\|_\infty > s) \leq \exp(-ks^{d-1})$ when $x \sim p_{\theta^*}$.
\end{condition}

This assumption is satisfied by several natural exponential families. For example, we have the following pertaining to the exponential families studied in \cite{pabbaraju2023provable}, proven in Appendix \ref{fact_proof}.
\begin{fact} \label{moitra_fact}
    If $T(x)$ contains only monomials of degree at most $d-1$ and $h(x) = \exp \left ( -\sum\limits_{i=1}^n x_i^{d} \right )$, then $p_{\theta^*}$ satisfies Assumption \ref{Bound_Assum} with $k=1$ and $C_t = nB+1$.
\end{fact}

\subsection{The Structure Graph of Exponential Families}
In order to define the structure and factor graph of a model, we will now introduce some additional notation pertaining to graphical models. The structure of a model is of particular interest in learning, as it encodes the conditional independence between variables in the model.

Given an exponential family with associated factor set $\K$ as above, the associated factor graph is a bipartite graph $G = ([n], \K, E)$ with an edge set
\[E = \{(i,k) \in [n] \times \K | i \in [k]\}. \eql{Factor_Graph}\]
We see from the above that an edge $(i,k)$ occurs when the variable $x_i$ appears in the associated basis function $f_k$. Thus, we can see that the neighborhood of any factor $k \in \K$ in the factor graph is precisely its support $\partial k$. Notice that the definition in \ref{eq:Factor_Graph} depends only on the set of basis functions $\K$ associated with our family of functions, and is independent of the true underlying model. The factor graph $G^* = ([n], \K^*, E^*)$ associated with a specific model $p_{\theta^*}$ in the family is the induced subgraph of $G$ obtained by taking $\K^* = \{k\in \K | \theta^*_k \neq 0\}$.

For any given factor graph $G$, the \textit{maximal factors} $\M_{\rm fac}$ consists of every factor whose neighborhood is not strictly contained within the neighborhood of another factor, 
\[\M_{\rm fac}(G) = \{k \in \K| \not \exists k' \in \K, \partial k \subset \partial k'\}.\] 
It is possible that multiple distinct maximal factors will have the same neighborhood. For example if the basis function $x_1 x_2$ corresponds to a maximal factor, then the basis function $x_1^2 x_2^2$ will as well, as they have the same neighborhood in the factor graph. It will be convenient to define \textit{maximal cliques} which correspond to the neighborhoods of all the maximal factors as
\[\M_{\rm cli}(G) = \{c \in \P([n])| \exists k \in \M_{\rm fac}(G), c = \partial k\}, \]
where $\P$ denotes the powerset.
We will call the span of a clique $c$ to be the set of all factors whose neighborhood is exactly $c$ denoted by $[c]_{\rm sp} = \{k \in \M_{\rm cli}| c = \partial k\}$. Finally, we can define the structure $\S$ of a model to be the set of maximal cliques in the models factor graph $G^*$, meaning
\[\S = \M_{\rm cli}(G^*).\]

\subsection{Score Matching}

Formally, the score matching loss for some parameter vector $\theta$ and sample $x$ is
\begin{align*}
    \L(\theta, x) &= \frac{1}{2} \|\nabla_x \log p_{\theta^*}(x) - \nabla_x \log p_\theta(x)\|^2.
\end{align*}
This defines the following loss function when averaged over the true distribution, $\L(\theta) = \E_{x \sim p_{\theta^*}} [\L_i(\theta,x)]$. 

This loss is computable only with knowledge of the true parameters. It can be shown (\cite{hyvarinen2005estimation}), under some mild assumptions on the decay of the true distribution, that $\L(\theta)$ is expressible in the following way up to a constant that does not depend on the optimization variables,
\begin{equation} 
    \L (\theta) = \E_{x \sim p_{\theta^*} } {\rm Tr} \nabla_x^2 \log p_\theta(x) + \frac{1}{2} \|\nabla_x \log p_\theta(x)\|^2 + C. \label{score_loss}
\end{equation}
 In this work, we will be interested primarily in analyzing the local score matching loss around each vertex $i \in [n]$, which we define as,
\begin{align*}
    \L_i(\theta, x) &:= \frac{\partial}{\partial^2 x_i} \log p_{\theta}(x) + \frac{1}{2} \left (\frac{\partial}{\partial x_i} \log p_{\theta}(x)\right )^2. \eql{score_loss_partial} 
\end{align*}
We let $\L_i(\theta) = \E_{x \sim p_{\theta^*}} [\L_i(\theta,x)]$. Further, for any collection of samples $x^{(1)}, \ldots x^{(M)}$, we let $\L_i(\theta, x^{(M)}) = \frac{1}{M} \sum\limits_{m=1}^M \L_i(\theta, x^{(m)})$.
 In the case of exponential families, this loss function can be computed by a quadratic optimization problem.

\subsection{Main Results}

We are now ready to formally state our result on structure learning, which we prove in Section \ref{family_struct}.
\begin{theorem}[Family Structure Learning] \label{fam_struc_thm}
Fix some exponential family $p_{\theta^*}$ with base measure $h(x) = 1$. For a fixed index $i \in [n]$, let $\hat \K$ be the maximal factors with $i$ as a neighbor in the family factor graph $G$, meaning $\hat \K = \{k \in \M_{\rm fac}(G)| i \in \partial k\}.$ Further, for $M$ independent samples $x_1, \ldots x_M \sim p_{\theta^*}$, let $\hat \theta$ be the minimizer of $\L_i(\theta, x^{(M)})$ subject to $\sum\limits_{k \in \K_j} |\theta_k| \leq B$ for every index $j$. There exists some $M^* = (d B C_t^d)^{O(d^2 w)}$ such that for every $\rho \geq 1$ and $\epsilon \leq 1$, we have that $(\theta^*_k - \hat \theta_k)^2 \leq \epsilon$ for every $k \in \hat \K$ with probability greater than $1 - \frac{1}{\rho n C_t}$ when $M \geq \rho \frac{n^{d+1} M^*}{\epsilon^2}$.
\end{theorem}
Notice that this theorem learns the parameters for \textit{family} structure, not the structure of the specific model. In particular, we only learn $k \in \M_{\rm fac}(G)$ the maximal factors for the \textit{family} factor graph $G$, not the model graph $G^*$. However, provided a sufficient number of samples, one can run an iterative algorithm to identify and remove cliques which are present in the family but not the particular model, allowing for the recovery of the \textit{model} structure. This idea is made precise in Theorem \ref{alg_thm}. 
Though we present the theorem here assuming the score-matching minimizer can be computed exactly, this may be intractable in practice. A version of this result  is presented in Appendix \ref{approx_sec} which provides recovery guarantees so long as $\hat \theta$ has a loss sufficiently close to the minimum's.

We state Theorem \ref{fam_struc_thm} assuming the base measure $h(x) = 1$. However, the result can be applied to any distribution satisfying Condition \ref{Bound_Assum} so long as its base measure can be encoded as a polynomial exponential family. In particular, this theorem applies directly to the family from Fact \ref{moitra_fact}, as the base measure consists of bounded degree monomials which do not supersede any maximal factors

When each basis function $f_k$ is linear in each variable, it is possible to show a stronger result recovering the entire parameter vector $\theta^*$ using score matching. In particular, we define a model $p_{\theta^*}$ to be a \textit{multi-linear exponential model} if $k_i \in \{0,1\}$ for every $k \in \K$ and every $i \in [n]$. Under this assumption, we have the following theorem proven in Section \ref{multi_lin_model}.
\begin{theorem}[Total Recovery of Multi-Linear Models] \label{multilinear_thm}
Suppose $p_{\theta^*}$ is a multi-linear exponential model with base measure $h(x) = \exp \left ( -\sum\limits_{i=1}^n x_i^{d} \right )$. For $M$ independent samples $x_1, \ldots x_M \sim p_{\theta^*}$, let $\hat \theta$ be the minimizer of $\L_i(\theta, x^{(M)})$ subject to $\sum\limits_{k \in \K_j} |\theta_k| \leq B$ for every index $j$. There exists some $M^* = (B C_t^d)^{O(d)}$ such that for every $\rho \geq 1$ and $\epsilon \leq 1$, $(\theta^*_k - \hat \theta_k)^2 \leq \epsilon$ for every $k \in \K_i$ with probability greater than $1 - \frac{1}{\rho n C_t}$ when $M \geq \rho \frac{n^{2d+1} M^*}{\epsilon^2}$.
\end{theorem}

\section{Curvature Analysis}\label{Curvature Analysis}

This section will focus on analyzing the curvature of the loss function $\L_i$ and deriving a few key properties which will be useful in our proofs. Bounding the curvature of the loss function allows us to show that a difference in loss implies a minimal distance between parameter vectors in euclidean space. This is the key ingredient in our proof of the finite sample bound.

Suppose $\hat \theta$ is some candidate parameter vector. We define $\Delta = \hat \theta - \theta^*$, $E_i(x, \theta) =  \frac{\partial}{\partial x_i} \log p_{\theta(x)} =  \frac{\partial}{\partial x_i} \sum\limits_{k \in \K} \theta_k f_k(x)$, and $\phi(y) = y^2/2.$ Notice that $E_i$ is linear in each $\theta_k$ term. We have that
\begin{align*} 
    \L_i(\theta^*) - \L_i(\hat \theta) &= \E_{x \sim p_{\theta^*}} \left [ \phi(E_i(x,\theta^*)) + \frac{\partial}{\partial^2 x_i}\log p_{\theta^*} \right ] - \E_{x \sim p_{\theta^*}} \left [ \phi(E_i(x,\hat \theta)) + \frac{\partial}{\partial^2 x_i}\log p_{\hat \theta} \right ]\\
    &= \E_{x \sim p_{\theta^*}} \left [ \phi(E_i(x,\theta^*)) \right ] - \E_{x \sim p_{\theta^*}} \left [ \phi(E_i(x,\hat \theta))  \right ] - \sum\limits_{k \in \K} \Delta_k \E_{x \sim p_{\theta^*}} \left [ \frac{\partial}{\partial^2 x_i} f_k(x) \right ].
\end{align*}
Further, we have that
\begin{align*}
    \<\Delta, \nabla \L_i( \theta^*)\> &= \E_{x \sim p_{\theta^*}}  \phi'(E_i(x,\theta^*))\left ( \sum\limits_{k \in \K} \frac{\partial}{\partial \theta_\k \partial x_i} \log p_{\theta^*}(x) \Delta_k \right ) - \sum\limits_{k \in \K} \Delta_k \E_{x \sim p_{\theta^*}} \left [ \frac{\partial}{\partial^2 x_i} f_k(x) \right ]\\
    &= \E_{x \sim p_{\theta^*}} \phi'(E_i(x,\theta^*)) E_i(x, \Delta) - \sum\limits_{k \in \K} \Delta_k \E_{x \sim p_{\theta^*}} \left [ \frac{\partial}{\partial^2 x_i} f_k(x) \right ]. 
\end{align*}
We define the curvature of a function $\psi$ as it's deviation from it's linear approximation at a point, $\delta \psi(\Delta, x) = \psi(x + \Delta) - \psi(x) - \<\Delta, \nabla \psi(x)\>$. Combining this definition with the previous two equations, we see that
\begin{align*}
    \delta \L_i(\Delta, \theta^*) &= \L_i(\hat \theta) - \L_i(\theta^*) - \<\Delta, \nabla \L_i( \theta^*)\>\\
    &= \E_{x \sim p_{\theta^*}} \left [ \phi(E_i(x,\hat \theta))  \right ] - \E_{x \sim p_{\theta^*}} \left [ \phi(E_i(x,\theta^*)) \right ] - \E_{x \sim p_{\theta^*}} \phi'(E_i(x,\theta^*)) E_i(x, \Delta)\\
    &= \E_{x \sim p_{\theta^*}} \left [ \phi(E_i(x,\theta^* + \Delta)) \right ] - \E_{x \sim p_{\theta^*}} \left [ \phi(E_i(x,\theta^*))  \right ] - \E_{x \sim p_{\theta^*}} \phi'(E_i(x,\theta^*)) E_i(x, \Delta)\\
    &= \E_{x \sim p_{\theta^*}} \delta \phi(E_i(x,\theta^*), E_i(x, \Delta))\\
    &= \frac{1}{2} \E_{x \sim p_{\theta^*}} E_i[x, \Delta]^2, \eql{curvature} 
\end{align*}
where the final equality follows from $\delta \phi(x, \Delta) = (x+ \Delta)^2/2 - x^2/2 - \Delta x = \Delta^2/2$. Let $p_t(x)$ denote the distribution $p_{\theta^*}(x)$ conditioned on $\|x\|_\infty \leq C_t$. By applying Condition \ref{Bound_Assum} to equation \ref{eq:curvature}, we see that
\[\frac{1}{2}\E_{x \sim p_{\theta^*}} E_i[x, \Delta]^2 \geq \frac{1}{4} \E_{x \sim p_{t}} E_i[x, \Delta]^2, \eql{tailbound}\]
which will provide a more convenient reference distribution to use. Showing a bound on the above in terms of $\Delta$ will be sufficient to prove our main results. To see why, suppose we have some collection of samples $x^{(M)}$ and that $\hat \theta$ is the minimizer of $\L_i(\theta, x^{(M)})$. Then, we have that
\begin{align*}
0 \geq \L_i(\hat \theta, x^{(M)}) - \L_i(\theta^*, x^{(M)})& = \delta \L_i(\Delta, \theta^*, x^{(M)}) + \<\Delta, \nabla_\theta \L_i(\theta^*, x^{(M)})\>\\
&\geq \delta \L_i(\Delta, \theta^*, x^{(M)}) - 2B \|\nabla_\theta \L_i(\theta^*, x^{(M)})\|_\infty. \eql{0-bound}
\end{align*}
In Appendix \ref{Curv_Conc_Sec}, we prove a high probability bound on $\|\nabla_\theta \L_i(\theta^*, x^{(M)})\|_\infty$ under Assumption \ref{Bound_Assum} and that $\delta \L_i(\Delta, \theta^*, x^{(M)})$ concentrates around its expectation using standard techniques. This culminates in Lemma \ref{main_lemma}, which allows us to derive our main theorems directly from a bound on $\E_{x \sim p_{t}} E_i[x, \Delta]^2$.  Thus, our primary technical challenge shifts to bounding the right hand side of Equation \eqref{eq:tailbound}.

\section{Recovery of Family Structure} \label{family_struct}


This section will be dedicated to proving Theorem \ref{fam_struc_thm}, a finite sample bound to recover a model's parameters corresponding to the \textit{model family's} maximal factors. Specifically, this section will work towards validating Equation \eqref{eq:struct_curv_bound}, which immediately implies Theorem \ref{fam_struc_thm} when combined with Lemma \ref{main_lemma}, a concentration result for the model curvature. In what follows, we will assume that we have a specific model $p_{\theta^*}$ parameterized by the vector $\theta^*$ that lies within an exponential model family as defined above.

\subsection{Grid Points}

We will now fix some specific maximal clique $c \ni i$. 
In order to obtain structure recovery, we will need to lower bound the expression in Equation $\eqref{eq:tailbound}$ in terms of $\Delta_k$ for each $k \in [c]_{\rm sp}$. To this end, we will first observe that
\[\E_{x \sim p_{t}} E_i[x, \Delta]^2 = \E_{x_{\setminus c} \sim p_{t}} [ \E_{x_c|x_{\setminus c} \sim p_{t}} [E_i[x, \Delta]^2]]. \eql{tower}\]
This section will be dedicated to finding some simple and separable distribution $q_{x_{\setminus c}}(x_c)$ which we can take the inner expectation with respect to instead.

Fix some value for $x_{\setminus c}$ satisfying $\|x_{\setminus c}\|_\infty \leq C_t$. Let $\rho(x_c|x_{\setminus c})$ be the probability density function for the distribution $x_c|x_{\setminus c} \sim p_{t}$ defined as
\[ \rho(x_c| x_{\setminus c}) = \frac{1}{Z(x_{\setminus c})} \exp \left (\sum\limits_{k \in \K} \theta^*_k f_k(x) - \sum\limits_{i \in c} x_i^{d} \right ) \1(\|x_c\|_\infty \leq C_t).\]
We let $\hat x_c = {\rm argmax}_{x_c} \rho(x_c|x_{\setminus c})$ and take $\gamma = d^2 C_t^{2d} B$. We will further take $l_i$ to be some set of integers such that $\frac{l_i}{\gamma} \leq (\hat x_c)_i \leq \frac{l_i + 1}{\gamma}$ for each $i \in c$. Take any alternative $x_c$ satisfying $\frac{l_i}{\gamma} \leq (x_c)_i \leq \frac{l_i + 1}{\gamma}$ and let $\epsilon = \hat x_c - x_c$. We see that
\begin{align*}
|f_k(\hat x) - f_k(x)| &= \left | \prod\limits_{i \not \in c} x_i^{k_i} \left ( \prod\limits_{i \in c} \hat x_i^{k_i} - \prod\limits_{i \in c} x_i^{k_i} \right ) \right | \\
&=  \prod\limits_{i \not \in c} |x_i|^{k_i} \left ( \prod\limits_{i \in c} (|\hat x_i| + |\epsilon_i|)^{k_i} - \prod\limits_{i \in c} |\hat x_i|^{k_i} \right ) \\
&\leq  \prod\limits_{i \not \in c} C_t^{k_i} \left ( \prod\limits_{i \in c} (C_t + \frac{1}{\gamma})^{k_i} - \prod\limits_{i \in c} C_t^{k_i} \right )\\
&\leq C_t^d ((C_t + \frac{1}{\gamma})^d - C_t^d)\\
&\leq \frac{d C_t^{2d}}{\gamma} + \frac{2^d C_t^{2d}}{\gamma^2}\\
&\leq \frac{1}{dB}.
\end{align*}
Thus, we conclude that
\begin{align*}
\frac{\rho(x_c|x_{\setminus c})}{\rho(\hat x_c| x_{\setminus c})} &= \exp \left ( \sum\limits_{k \in \K} \theta^*_k (f_k(x) - f_k(\hat x)) \right ) \\
&\geq \exp \left ( -\sum\limits_{k \in \K} \theta^*_k |f_k(x) - f_k(\hat x)|  \right ) \\
&\geq \exp \left ( \frac{-1}{dB} (\sum\limits_{k \in \K, \partial k \ni i} \theta^*_k )   \right )\\
&\geq e^{-1}.
\end{align*}
Next, we see that 
\[1 = \int_{\|x'_c\|_\infty \leq C_t} \rho(x'_c | x_{\setminus c}) dx \leq \int_{\|x'_c\|_\infty \leq C_t} \rho(\hat x_c | x_{\setminus c}) dx = C_t^d \rho(\hat x_c| x_{\setminus c}),\]
which in turn implies that $\rho(x_c|x_{\setminus c}) \geq \frac{1}{e C_t^d}$. Therefore, we define our \textit{centering distribution} $q_{x_{\setminus c}}(x_c)$ to be the uniform probability distribution over all $x_c$ satisfying $\frac{l_i}{\gamma} \leq (x_c)_i \leq \frac{l_i + 1}{\gamma}$, so that we have
\[\E_{x_{\setminus c} \sim p_{t}} [ \E_{x_c|x_{\setminus c} \sim p_{t}} [E_i[x, \Delta]^2]] \geq \frac{1}{eC_t^d} \E_{x_{\setminus c} \sim p_{t}} [ \E_{x_c \sim q_{x_{\setminus c}}} [E_i[x, \Delta]^2]], \eql{Lower_Expectation}\]
fulfilling our initial goal. For the rest of the section, we will refer to $q_{x_{\setminus c}}$ simply as $q$ for simplicity.

\subsection{Bounds for Centered Basis Function}

In the course of the following proof, we will want to work with basis functions which are centered with respect to our centering distributions $q$. For each factor $k$ where $\partial k = c$, we will define a corresponding \textit{mostly centered basis function} $h_{i,k}$ as
\[h_{i,k}(x) = k_i x_i^{k_i - 1} \sum\limits_{r \in \P(c \setminus i)} (-1)^{|r|} \E_{x_r \sim q^{|r|}}[f_k(x)] = 
k_i x_i^{k_i-1}\prod\limits_{i \in c\setminus i} (x_i^{k_i} - \E_{x_i \sim q}[x_i^{k_i}]).\]
We do not center around the variable $x_i$ above as it is not necessary for the remainder of the proof, and doing so would force $h_{i,k} = 0$ whenever $k_i = 1$.
We introduce the following Lemma.
\begin{lemma} \label{NPC_Lemma}
    For any reference index $i$, any maximal clique $c \ni i$, the Nonsingular Parametrization of Cliques constant $B_{\rm NPC}$, defined as
    \[B_{\rm NPC} = \min\limits_{\|x\|_2 = 1} \E_q \left [ \left ( \sum\limits_{k \in [c]_{\rm sp}} x_k h_k(x) \right )^2 \right ],\]
    will satisfy $B_{\rm NPC} \geq \exp(- O(d^2 |c| \log \gamma))$.
\end{lemma}
\begin{proof}
We first define the matrix $M_{k,k'} = \E_q[h_k h_{k'}]$ for each $k,k' \in [c]_{\rm sp}$, and observe that
\[\left ( \sum\limits_{k \in [c]_{\rm sp}} x_k h_k(x) \right )^2 = x_k^T M x_k,\]
which implies $B_{\rm NPC} = \lambda_{\rm min}(M)$. For each index $j \neq i$, we now define a new $d\times d$ matrix $A^j$ where 
\begin{align*}
A^j_{m,n} &= \E_{x \sim q} [(x^m - \E_q[x^m])(x^n - \E_q[x^n])]\\
&= \frac{(l_j + 1)^{m+n+1} - l_j^{m+n+1}}{\gamma^{m+n+1}(m+n+1)} - \frac{((l_j + 1)^{m+1} - l_j^{m+1})((l_j + 1)^{n+1} - l_j^{n+1})}{\gamma^{m+1}\gamma^{n+1} (m+1)(n+1)}.
\end{align*}
This matrix $A^j$ will have a few key properties. First, $A^j$ is positive-definite, as\\ 
$y^T A^j y = \E_q[(\sum_j y_j (x^j - \E_q[x^j])^2] > 0$.
Second, $\lambda_{\rm max}(A_j) \leq d$ since each entry in $A^j$ is bounded above by 1. Finally, $\gamma^{2d+1}(2d+1)!(d+1)!A^j$ has only integer entries, since $\gamma$ is an integer by the assumption that $C_t$ and $B$ are integers. And, since integer matrices have integer determinants, this implies it has a determinant of at least 1 when combined with the positive-definiteness. Thus, since we know that $\gamma > d$, we have $\det(A^j) \geq \exp(-O(d^2 \log \gamma ))$, ultimately implying
\[\lambda_{\rm min}(A^j) \geq \frac{\det(A^j)}{\lambda_{\rm max}(A^j)^{d-1}} \geq \frac{\det(A^j)}{d^{d-1}} = \exp(-O(d^2 \log \gamma )).\]
For index $i$, we will define the matrix $A^i$ such that $A^i_{m,n} = \E_q[x^{m+n}]$, and we follow an analogous argument to show $\lambda_{\rm min}(A^i) \geq \exp(-O(d^2 \log \gamma))$.
Fixing any particular $k$ and $k'$, we see that
\begin{align*}
    M_{k,k'} =& \E_{x \sim q}[h_k(x) h_{k'}(x)]\\
    =& \E_q \left [ x_i^{k_i + k'_i - 2}\prod_{j \in c\setminus i}(x_j^{k_j} - \E_q[x_j^{k_j}])(x_j^{k'_j} - \E_q[x_j^{k'_j}]) \right ]\\
    =&\E_q[x_i^{k_i + k'_i - 2}] \prod_{j \in c\setminus i} \E_q \left [(x_j^{k_j} - \E_q[x_j^{k_j}])(x_j^{k'_j} - \E_q[x_j^{k'_j}]) \right ]\\
    =& \prod_{j \in c} A^j_{k_j, k'_j}.
\end{align*}
Thus, $M$ will be a proper sub-matrix of $\otimesL\limits_{j \in c} A^j$, implying that 
\[\lambda_{\rm min}(M) \geq \prod\limits_{j \in c} \lambda_{\rm min}(A^j) = \exp(-O(d^2 |c| \log \gamma)).\]

\end{proof}
\subsection{Curvature Bounds for Maximal Cliques}

This section will be dedicated to completing our bound on the expected curvature. For convenience, we will let $f_{i,k}(x) = \frac{\partial}{\partial x_i} f_k(x)$. We know $f_{i,k}(x) = h_{i,k}(x) + R_{i,k}(x)$ for $\partial k = c$, where
\[R_{i,k}(x) = \sum\limits_{r \in \P(c \setminus i) \setminus \emptyset} (-1)^{|r|} \E_{x_r \sim q}[ f_{i,k}(x)]. \]
Picking up from Equation \eqref{eq:Lower_Expectation}, we see that

\begin{align*}
     \E_{x_c \sim q} \left [ \left (\sum\limits_{k \in \K_i} \Delta_k f_{i,k}(x) \right )^2 \right ] 
    &= \E_{x_c \sim q} \left [ \left ( \sum\limits_{k \in [c]_{\rm sp}} \Delta_k  h_{i,k}(x) \right )^2 \right ] \eql{num1}\\
    &+ 2\sum\limits_{k \in [c]_{\rm sp}} \sum\limits_{l' \in \K_i\setminus [c]_{\rm sp}} \Delta_k \Delta_{k'} \E_{x_c \sim q}[h_{i,k}(x) f_{i,k'}(x)] \eql{num2} \\
    &+ \sum\limits_{k,k' \in [c]_{\rm sp}} \Delta_k \Delta_{k'} \E_{x_c \sim q} [h_{i,k}(x) R_{i,k'}(x)] \eql{num3}\\
    &+ \E_{x_c \sim q} \left [ \left ( \sum\limits_{k \in \K_i\setminus [c]_{\rm sp}} \Delta_k f_{i,k}(x) + \sum\limits_{k \in [c]_{\rm sp}} \Delta_k R_{i,k}(x) \eql{num4} \right )^2 \right ].
\end{align*}
We now look at the contribution from each line of the above. In Equation \eqref{eq:num2}, for any given $k$ and $k'$, there must be some index $j \neq i$ where $j \in c$ and $j \not \in \partial k'$. Since $f_{i,k'}$ will not depend on $x_j$, we see
\[\E_{x_c \sim q}[h_{i,k}(x) f_{i,k'}(x)] = \E_{x_{c \setminus j} \sim q} [f_{i,k'}(x) \E_{x_j \sim q}[h_{i,k}(x) ]] = 0,\]
meaning the contribution from line \eqref{eq:num2} is 0. Expanding the definition of $R_{i,k'}$ in line \eqref{eq:num3}, we see
\[\E_{x_c \sim q} [h_{i,k}(x) R_{i,k'}(x)] = -\sum\limits_{r \in \P(c \setminus i) \setminus \emptyset} (-1)^{|r|} \E_{x_c \sim q} [h_{i,k}(x) \E_{x_r \sim q}[f_{i,k'}(x)]] = 0,\]
since $r$ always contains an index $j \neq i$ which $\E_{x_r \sim q}[f_{i,k'}(x)]$ will not depend on, implying the contribution from line \eqref{eq:num3} is also 0. Since Equation \eqref{eq:num4} is the expectation of a squared value, its contribution must be strictly non-negative. Thus, we conclude that
\begin{align*}
\E_{x_c \sim q} \left [ \left (\sum\limits_{k \in \K_i} \Delta_k f_{i,k}(x) \right )^2 \right ] 
&\geq \E_{x_c \sim q} \left [ \left ( \sum\limits_{k \in [c]_{\rm sp}} \Delta_k  h_{i,k}(x) \right )^2 \right ]
\geq B_{\rm NPC} \sum\limits_{k \in [c]_{\rm sp}} \Delta_k^2,
\end{align*}
where $B_{\rm NPC}$ is defined in Lemma \ref{NPC_Lemma}. Combining the above with Equations \eqref{eq:tower} and \eqref{eq:Lower_Expectation}, we conclude that
\[\E_{x \sim p_t} E_i(x, \Delta)^2 \geq \frac{B_{\rm NPC}}{eC_t^d} \sum\limits_{k \in [c]_{\rm sp}} \Delta_k^2. \eql{struct_curv_bound}\]
The above equation, when combined with Lemma \ref{main_lemma}, immediately implies Theorem \ref{fam_struc_thm}.

\section{Recovery of Model Structure} \label{model_struct}

In this section, we will present an approach to strengthen Theorem \ref{fam_struc_thm} to allow for the identification of the structure graph for a specific model and recover the parameter weights for the corresponding maximal factors. Suppose we have some $2\sqrt{\epsilon}$-lower bound with $\epsilon \leq 1$ on the strength of parameters associated with any maximal factor of the model, meaning that $\theta^*_k \geq 2\sqrt{\epsilon}$ for every $k \in \M_{\rm fac}(G^*)$. Then, for sufficiently many samples $M$, Algorithm \ref{alg} will be able to recover both the structure $\S$ of the underlying model and the associated parameter $\theta^*_k$ for the maximal factors $k$ of the model. In particular, we have the following theorem.

\begin{theorem} [Iterative Model Structure Learning] \label{alg_thm}
The output $(\hat \S, \hat \theta_i)$ of Algorithm \ref{alg} satisfies $\hat \S = \S$ and $\left ((\hat \theta_i)_k - \theta^*_k \right )^2 \leq \epsilon$ with probability greater than $1 - \frac{1}{\rho C_t}$ as long as $M \geq \rho \frac{w n^{2d+1} M^*}{\epsilon^2}$, where $w$ is the interaction order of the model family and $M^* =(d B C_t^d)^{O(d^2 w)}$ is defined as in Theorem \ref{fam_struc_thm}.
\end{theorem}

\begin{proof}
By Theorem \ref{fam_struc_thm}, we know that each $\theta_i^s$ defined on line 8 of Algorithm \ref{alg} will satisfy 
\[\left ((\hat \theta_i^s)_k - \theta^*_k\right )^2 \leq \epsilon \eql{alg_cond}\]
for each $k \in \hat \K_i$ with probability at least $1 - \frac{1}{\rho n w C_t}$, where $\hat \K_i$ is defined on line 7. Thus, by the union bound, the probability that Equation \eqref{eq:alg_cond} is satisfied by every $\hat \theta_i^s$ is at least $1 - \frac{1}{\rho C_t}$, since there are $w$ different values of $s$ and $n$ unique values of $i$. 

Now, suppose Equation \eqref{eq:alg_cond} is satisfied for every $\hat \theta^s_i$. Take some clique $c \in \M_{\rm cli}(G^s)$ which is not in $\M_{\rm cli}(G^*)$. For each $k \in [c]_{\rm sp},$ $\theta^*_k = 0$, meaning that $\hat (\theta_i^s)_k^2 \leq \epsilon$. Thus, every element of $[c]_{\rm sp}$ will be absent from $\K^{s+1}$, meaning that $c \not \in \M_{\rm cli}(G^{s+1}).$ Therefore, every clique $c \in \M_{\rm cli}(G^s)$ with $|c| > m-s$ must also be in $\M_{\rm cli}(G^*)$, since it would have otherwise been removed in the previous iteration. We conclude that $\S = \M_{\rm cli}(G^w) = \hat S$ and $\M_{\rm fac}(G^w)\supseteq \M_{\rm fac}(G^*)$, proving the theorem.   
\end{proof}

\begin{algorithm}[t]
\caption{An Iterative Algorithm for Model Structure Recovery}
\begin{algorithmic}[1]
\label{alg}
\STATE Let $x_1, \ldots x_M \sim p_{\theta^*}$ be independent samples
\STATE Let $G$ be the bipartite factor graph for the model family
\STATE Let $\K^0 \gets \K$ be the current set of potential factors
\FOR{$s = 0, \ldots w$}
\STATE Let $G^s \gets G[\K^s]$ be the subgraph of $G$ induced by $\K^s$
\FOR{$i \in [n]$}
\STATE Let $\hat \K_i \gets \{k \in \M_{\rm fac}(G^s) | i \in \partial k\}$
\STATE Let $\hat \theta_i^s \gets {\rm argmin}_\theta (\L_i(\theta, x^{(M)}))$ subject to $\sum\limits_{k \in \K_j} |\theta_k| \leq B$ for all $j$ with factor set $\K^s$
\ENDFOR
\STATE Let $\mathcal{N}^s \gets \emptyset$
\FOR{$c \in \M_{\rm cli}(G^s)$}
\IF{$\not \exists k \in [c]_{\rm sp}$ where $\hat (\theta_i^s)_k > \epsilon$ for some $i$}
\STATE $\mathcal{N}^s \gets \mathcal{N}^s \cup [c]_{\rm sp}$
\ENDIF
\ENDFOR
\STATE Let $\K^{s+1} \gets \K^s \setminus \mathcal{N}^s$ be the new set of potential factors
\ENDFOR
\RETURN $\hat \S = \M_{\rm cli}(G^w)$ and $\hat \theta_i = \hat \theta_i^{w}$
\end{algorithmic}
\end{algorithm}

\section{Total Recovery of Multi-linear Models} \label{multi_lin_model}

This section will be dedicated to proving Theorem \ref{multilinear_thm}, a finite sample bound for the total recovery of multi-linear model. As in Section \ref{family_struct}, we will prove a lower bound on $\E_{x \sim p_t} E_i(x, \Delta)^2$ in Equation \eqref{eq:Final_var} which immediately implies Theorem \ref{multilinear_thm} when combined with Lemma \ref{main_lemma}.
We will now assume that $p_{\theta^*}$ is a multi-linear model, meaning that for each $k \in \K$, $k_j \in \{0,1\}$ for every $j \in [n]$, with base measure $h(x) = \exp \left ( -\sum\limits_{i=1}^n x_i^{d} \right )$. We will fix both some index $i$ and a second index $j \neq i$ and continue from Equation \eqref{eq:tailbound}.
By the Law of Conditional Variances, we have that
\begin{align*}
\E_{x \sim p_t} E_i(x, \Delta)^2 \geq \Var_{x \sim p_t}(E_i(x, \Delta)) 
&\geq \E_{x_{\setminus j} \sim p_t} \Var_{x_j | x_{\setminus j} \sim p_t} E_i(x, \Delta)\\
&= \E_{x_{\setminus j} \sim p_t} \left ( \sum\limits_{k \ni j, i} \Delta_k x_{k\setminus i,j} \right )^2 \Var_{x_j | x_{\setminus j} \sim p_t} x_j.
\end{align*}
We now observe that the variable $x_j|x_{\setminus j} \propto \exp(\eta x_j - x_j^{d})$ on $[-C_t, C_t]$ for some $\eta$ dependent on $x_{\setminus j}$ and $\theta^*$. In particular, based on the structure of the function, we know $|\eta| \leq B C_t^{d-1}$. Thus, by applying Lemma \ref{1d_var}, a bound on the variance of $\exp(\eta x_j - x_j^{d})$ in terms of $\eta$, we have
\[\E_{x \sim p_t} E_i(x, \Delta)^2 \geq 
\left ( \frac{1}{2e B C_t^{d-1}} \right )^2
\E_{x_{\setminus j} \sim p_t} \left ( \sum\limits_{k \ni j, i} \Delta_k x_{k\setminus i,j} \right )^2.\]
We can generalize this argument. Take any index set $\j \not \ni i$, and any index $l \not \in \j \cup \{i\}$. We will have
\begin{align*}
    \E_{x_{\setminus \j} \sim p_t} \left ( \sum\limits_{k \ni \j, i} \Delta_k x_{k\setminus \{i,\j\}} \right )^2 
    &\geq \E_{x_{\setminus \j, l} \sim p_t} \Var_{x_l | x_{\setminus \j, l} \sim p_t} \left (  \sum\limits_{k \ni \j, i} \Delta_k x_{k\setminus \{i,\j\} }  \right ) \\
    &= \E_{x_\j \sim p_t} \left [ \E_{x_{\setminus \j, l} | x_{\j} \sim p_t} \left ( \sum\limits_{k \ni \j,i,l} \Delta_k x_{k \setminus i,l,\j} \right ) \Var_{x_l | x_{\setminus l} \sim p_t}(x_l)  \right ]\\
    &\geq  \left ( \frac{1}{2e B C_t^{d-1}} \right )^2 \E_{x_{\setminus \j, l} \sim p_t} \left ( \sum\limits_{k \ni \j,i,l} \Delta_k x_{k \setminus i,l,\j} \right ).
\end{align*}
Applying this logic repeatedly, we can see that for any $k$ where $i \in \partial k$,
\[\E_{x \sim p_t} E_i(x, \Delta)^2 \geq \left ( \frac{1}{2e B C_t^{d-1}} \right )^{2(d-1)} \Delta_k^2. \eql{Final_var}\] 
The proof of Theorem \ref{multilinear_thm} follows from plugging Equation \eqref{eq:Final_var} into our curvature Lemma \ref{main_lemma}.

\section{Conclusion} \label{Conclusion}

We have provided the first finite sample complexity bound for learning the structure of exponential families parametrized by bounded-degree polynomials using score matching. The sample bounds scale polynomially in the the model dimension and thus establish strong statistical merits of score matching in addition to its known computational properties. Information-theoretic bounds for recovery of exponential families with continuous variables are currently unknown. An important open question left for future exploration includes deriving such a lower bound, and comparing the rates to the scalings derived in this work. Our work establishes the properties of the optimal solution to the convex score matching objective under finite samples. It would be useful to complement our sample complexity analysis with the convergence analysis of  gradient descent methods and establish the optimal rates. Finally, it would be interesting to extend our analysis to the families of distributions with general parametrization of the energy function beyond polynomials.

\subsection*{Acknowledgment}
 This work has been supported by the U.S. Department of Energy/Office of Science Advanced Scientific Computing Research Program.

\bibliography{main}

@article{azangulov2024convergence,
  title={Convergence of diffusion models under the manifold hypothesis in high-dimensions},
  author={Azangulov, Iskander and Deligiannidis, George and Rousseau, Judith},
  journal={arXiv preprint arXiv:2409.18804},
  year={2024}
}

@inproceedings{oko2023diffusion,
  title={Diffusion models are minimax optimal distribution estimators},
  author={Oko, Kazusato and Akiyama, Shunta and Suzuki, Taiji},
  booktitle={International Conference on Machine Learning},
  pages={26517--26582},
  year={2023},
  organization={PMLR}
}

@inproceedings{tang2024adaptivity,
  title={Adaptivity of diffusion models to manifold structures},
  author={Tang, Rong and Yang, Yun},
  booktitle={International conference on artificial intelligence and statistics},
  pages={1648--1656},
  year={2024},
  organization={PMLR}
}

@article{yakovlev2025implicit,
  title={Implicit score matching meets denoising score matching: improved rates of convergence and log-density Hessian estimation},
  author={Yakovlev, Konstantin and Markovich, Anna and Puchkin, Nikita},
  journal={arXiv preprint arXiv:2512.24378},
  year={2025}
}

@inproceedings{yakovlev2025generalization,
  title={Generalization error bound for denoising score matching under relaxed manifold assumption},
  author={Yakovlev, Konstantin and Puchkin, Nikita},
  booktitle={The Thirty Eighth Annual Conference on Learning Theory},
  pages={5824--5891},
  year={2025},
  organization={PMLR}
}

@article{sriperumbudur2017density,
  title={Density estimation in infinite dimensional exponential families},
  author={Sriperumbudur, Bharath and Fukumizu, Kenji and Gretton, Arthur and Hyv{\"a}rinen, Aapo and Kumar, Revant},
  journal={Journal of Machine Learning Research},
  volume={18},
  number={57},
  pages={1--59},
  year={2017}
}

@article{ren2021learning,
  title={Learning continuous exponential families beyond gaussian},
  author={Ren, Christopher X and Misra, Sidhant and Vuffray, Marc and Lokhov, Andrey Y},
  journal={arXiv preprint arXiv:2102.09198},
  year={2021}
}

@inproceedings{shah2021learning,
  title={On learning continuous pairwise markov random fields},
  author={Shah, Abhin and Shah, Devavrat and Wornell, Gregory},
  booktitle={International conference on artificial intelligence and statistics},
  pages={1153--1161},
  year={2021},
  organization={PMLR}
}

@article{shah2021computationally,
  title={A computationally efficient method for learning exponential family distributions},
  author={Shah, Abhin and Shah, Devavrat and Wornell, Gregory},
  journal={Advances in neural information processing systems},
  volume={34},
  pages={15841--15854},
  year={2021}
}

@article{wu2019sparse,
  title={Sparse logistic regression learns all discrete pairwise graphical models},
  author={Wu, Shanshan and Sanghavi, Sujay and Dimakis, Alexandros G},
  journal={Advances in Neural Information Processing Systems},
  volume={32},
  year={2019}
}

@article{koehler2022statistical,
  title={Statistical efficiency of score matching: The view from isoperimetry},
  author={Koehler, Frederic and Heckett, Alexander and Risteski, Andrej},
  journal={arXiv preprint arXiv:2210.00726},
  year={2022}
}

@article{pabbaraju2023provable,
  title={Provable benefits of score matching},
  author={Pabbaraju, Chirag and Rohatgi, Dhruv and Sevekari, Anish Prasad and Lee, Holden and Moitra, Ankur and Risteski, Andrej},
  journal={Advances in Neural Information Processing Systems},
  volume={36},
  pages={61306--61326},
  year={2023}
}

@article{vuffray2022efficient,
  title={Efficient learning of discrete graphical models},
  author={Vuffray, Marc and Misra, Sidhant and Lokhov, Andrey},
  journal={Advances in Neural Information Processing Systems},
  volume={33},
  pages={13575--13585},
  year={2020}
}

@article{hyvarinen2005estimation,
  title={Estimation of non-normalized statistical models by score matching.},
  author={Hyv{\"a}rinen, Aapo},
  journal={Journal of Machine Learning Research},
  volume={6},
  number={4},
  year={2005}
}

@inproceedings{hinton2002training,
  title        = {Training Products of Experts by Minimizing Contrastive Divergence},
  author       = {Hinton, Geoffrey E.},
  booktitle    = {Neural Computation},
  year         = {2002}
}

@article{wainwright2008graphical,
    author = {Wainwright, Martin J. and Jordan, Michael I.},
    title = {Graphical Models, Exponential Families, and Variational Inference},
    journal = {Foundations and Trends in Machine Learning},
    volume = {1},
    number = {1-2},
    pages = {1-305},
    year = {2008},
    month = {12},
    issn = {1935-8237},
    doi = {10.1561/2200000001},
    url = {https://doi.org/10.1561/2200000001},
    eprint = {https://www.emerald.com/ftmal/article-pdf/1/1-2/1/11153930/2200000001en.pdf},
}

@article{ravikumar2010high,
  title        = {High-Dimensional Ising Model Selection Using $\ell_1$-Regularized Logistic Regression},
  author       = {Ravikumar, Pradeep and Wainwright, Martin J. and Lafferty, John D.},
  journal      = {The Annals of Statistics},
  volume       = {38},
  number       = {3},
  pages        = {1287--1319},
  year         = {2010}
}

@misc{vuffray2016interaction,
  title        = {Interaction Screening: Efficient and Robust Learning of Ising Models},
  author       = {Vuffray, Marc and Lokhov, Andrey Y. and Misra, Sidhant and Chertkov, Michael},
  howpublished = {arXiv preprint arXiv:1602.07014},
  year         = {2016}
}

@article{hyvarinen2007some,
  title        = {Some Extensions of Score Matching},
  author       = {Hyv{\"a}rinen, Aapo},
  journal      = {Computational Statistics \& Data Analysis},
  volume       = {51},
  number       = {5},
  pages        = {2499--2512},
  year         = {2007}
}

@inproceedings{song2019generative,
  title        = {Generative Modeling by Estimating Gradients of the Data Distribution},
  author       = {Song, Yang and Ermon, Stefano},
  booktitle    = {Advances in Neural Information Processing Systems (NeurIPS)},
  year         = {2019}
}

@inproceedings{song2021scorebased,
  title        = {Score-Based Generative Modeling through Stochastic Differential Equations},
  author       = {Song, Yang and Sohl-Dickstein, Jascha and Kingma, Diederik P. and Kumar, Abhishek and Ermon, Stefano and Poole, Ben},
  booktitle    = {International Conference on Learning Representations (ICLR)},
  year         = {2021}
}

@inproceedings{Klivans2017,
  title={Learning graphical models using multiplicative weights},
  author={Klivans, Adam and Meka, Raghu},
  booktitle={2017 IEEE 58th Annual Symposium on Foundations of Computer Science (FOCS)},
  pages={343--354},
  year={2017},
  organization={IEEE}
}

@inproceedings{lee2023convergence,
  title        = {Convergence of Score-Based Generative Modeling for General Data Distributions},
  author       = {Lee, Holden and Lu, Jianfeng and Tan, Yixin},
  booktitle    = {Proceedings of The 34th International Conference on Algorithmic Learning Theory},
  series       = {Proceedings of Machine Learning Research},
  volume       = {201},
  pages        = {946--985},
  year         = {2023}
}

@misc{gao2023wasserstein,
  title        = {Wasserstein Convergence Guarantees for a General Class of Score-Based Generative Models},
  author       = {Gao, Xuefeng and Nguyen, Hoang M. and Zhu, Lingjiong},
  howpublished = {arXiv preprint arXiv:2311.11003},
  year         = {2023}
}

@misc{bruno2025wasserstein,
  title        = {Wasserstein Convergence of Score-based Generative Models under Semiconvexity and Discontinuous Gradients},
  author       = {Bruno, Stefano and Sabanis, Sotirios},
  howpublished = {arXiv preprint arXiv:2505.03432},
  year         = {2025}
}

@misc{montanari_computational_2015,
    title = {Computational {Implications} of {Reducing} {Data} to {Sufficient} {Statistics}},
    url = {http://arxiv.org/abs/1409.3821},
    doi = {10.48550/arXiv.1409.3821},
    urldate = {2025-08-13},
    publisher = {arXiv},
    author = {Montanari, Andrea},
    month = jul,
    year = {2015},
    note = {arXiv:1409.3821},
}

@article{darmois1935lois,
  title={Sur les lois de probabilit{\'e}a estimation exhaustive},
  author={Darmois, Georges},
  journal={CR Acad. Sci. Paris},
  volume={260},
  number={1265},
  pages={85},
  year={1935}
}

@article{koopman1936distributions,
  title={On distributions admitting a sufficient statistic},
  author={Koopman, Bernard Osgood},
  journal={Transactions of the American Mathematical society},
  volume={39},
  number={3},
  pages={399--409},
  year={1936},
  publisher={JSTOR}
}

@inproceedings{pitman1936sufficient,
  title={Sufficient statistics and intrinsic accuracy},
  author={Pitman, Edwin James George},
  booktitle={Mathematical Proceedings of the cambridge Philosophical society},
  volume={32},
  issue={4},
  pages={567--579},
  year={1936},
  organization={Cambridge University Press}
}

\appendix

\section{Distribution Results}

This section is dedicated to the proof of two useful technical lemmas.

\begin{lemma} \label{Ratio Lemma}
    Let $\alpha, \beta$ be two probability distributions supported on an interval $S$ with pdfs $f_\alpha, f_\beta$ respectively. Suppose there exists some measurable set $A$ where for any $x \in A$ and $y \in S \setminus A$ we have that $\frac{f_\alpha(x)}{f_\alpha(y)} \geq \frac{f_\beta(x)}{f_\beta(y)}$. Then, $\Pr(\alpha \in A) \geq \Pr(\beta \in A)$.
\end{lemma}
\begin{proof}
    Rearranging, we have that $f_\alpha(x) f_\beta(y) \geq f_\beta(x) f_\alpha(y)$ for any $x \in A, y \in S \setminus A$. Thus we can compute 
    \begin{align*}
        \Pr(\alpha \in A) &= \int_{x \in A} f_\alpha(x) dx \\
        &= \int_{x \in A} \int_{y \in S} f_\alpha(x) f_\beta(y) dy dx\\
        &\geq \int_{x \in A} \left ( \int_{y \in S \setminus A} f_\beta(x) f_\alpha(y) dy + \int_{y \in A} f_\alpha(x) f_\beta(y) dy \right ) dx\\
        &= \Pr(\beta \in A) \Pr(\alpha \in s \setminus A) + \Pr(\alpha \in A) \Pr(\beta \in A)\\
        &= \Pr(\beta \in A).
    \end{align*}
\end{proof}

\begin{lemma} \label{1d_var}
    Let $x_1 \propto \exp(\eta x - x^{d+1})$ be a random variable supported on $[-B, B]$ for some $d \geq 1$, $\eta \geq 4$ and some $B \geq 1$. Then,
    \[\Var(x_1) \geq \frac{1}{4 e^2 \eta^2}.\]
\end{lemma}

\begin{proof}
The general approach for this proof will be as follows. We will first lower bound the variance as a function of the probability $x_1$ falls within a certain region. We will then find a chain of variables for which this probability is decreasing, until we reach a variable for which we can easily calculate this probability.

    Assume WLOG that $\eta >0$. We first observe that $\frac{d}{dx} \exp(\eta x - x^{d+1}) = 0 \implies x = \left ( \frac{\eta}{d+1} \right )^{d}$, so the variable $x_1$ has a mode of $\left ( \frac{\eta}{d+1} \right )^{d}$. Let $D = \min \left (B, \left ( \frac{\eta}{d+1} \right )^{d} \right )$. We observe that $D \geq 3/\eta$, since $B = 1 > 3/4 \geq 3/\eta$ and $\left (\frac{\eta}{d+1} \right )^{d}  \geq \left (\frac{4}{d+1} \right )^{d} \geq 0.9 > 3/4 \geq 3/\eta$ for any $d \geq 1$.
    
    From the definition of variance, we know that $\Var(x_1) \geq \frac{1}{\eta^2} \Pr(|x_1 - \E[x_1]| \geq \frac{1}{\eta})$. Since the pdf of $x_1$ is strictly increasing in the range $[-D, D]$, we can conclude that $\Pr(|x_1 - \E[x_1]| \geq \frac{1}{\eta}) \geq \Pr(|x_1| \leq D - 2/\eta)$. This leads us to our first inequality,
    \[\Var(x_1) \geq \frac{1}{\eta^2} \Pr(|x_1| \leq D - 2/\eta). \eql{x1_bound}\]
    We now introduce a new variable $x_2 \propto \exp(\eta x - x^{d+1})$ with $x_2 \in [-\infty, B]$. Since $x_2$ will have a larger normalizing coefficient that $x_1$, we clearly have that
    \[\Pr(|x_1| \leq D - 2/\eta) \geq \Pr(|x_2| \leq D - 2/\eta). \eql{x2_bound}\]
    We now need to look at the shape of the distribution for $x_2$. Let $f_{x_2}(x) = \frac{\1(x < B)}{Z_2}\exp(\eta x - x^{d+1})$ denote a pdf function for $x_2$, where $Z_2$ is an appropriate normalizing coefficient. For any fixed $c$, we have that
    \begin{align*}
        f_{x_2}(D + c) &= \frac{\1(D+c < B)}{Z_2}\exp(\eta (D+c) - (D+c)^{d+1})\\
        &\leq \frac{1}{Z_2} \exp \left ( \eta D + \eta c - \sum\limits_{j=0}^{d+1} D^j c^{d+1-j} \right )\\
        &= \frac{1}{Z_2} \exp \left ( \eta D  - D^{d+1} + \eta c - (d+1)D^d c - \sum\limits_{j=2}^{d+1} D^j c^{d+1-j} \right )\\
        &= \frac{1}{Z_2} \exp \left ( \eta D  - D^{d+1}  - \sum\limits_{j=2}^{d+1} D^j c^{d+1-j} \right )\\
        &\leq \frac{1}{Z_2} \exp \left ( \eta D  - D^{d+1}  - \sum\limits_{j=2}^{d+1} (-1)^j D^j c^{d+1-j} \right )\\
        &= \frac{1}{Z_2} \exp \left ( \eta D  - D^{d+1} - \eta c + (d+1)D^d c  - \sum\limits_{j=2}^{d+1} (-1)^j D^j c^{d+1-j} \right )\\
        &= \frac{1}{Z_2} \exp \left ( \eta(D - c) - (D-c)^{d+1} \right )\\
        &= f_{x_2}(D - c).
    \end{align*}
    Thus, $\Pr(x_2 > D) \leq \Pr(x_2 < D)$. This inspires the creation of a third variable $x_3 \propto \exp(\eta x_3- x_3^{d+1})$ with $x_3 \in [-\infty, D]$. Since $\Pr(x_2 > D) \leq \Pr(x_2 < D)$, we will have that $Z_3 \leq 2 Z_2$, where $Z_3$ is the normalizing coefficient for $x_3$. Thus, we have that
    \[\Pr(|x_2| \leq D - 2/\eta) \geq \frac{1}{2}\Pr(|x_3| \leq D - 2/\eta). \eql{x3_bound}\]
    We are now ready to define our final variable $x_4 \propto \exp(\eta x_4)$ with $x_4 \in [-\infty, D]$. We first compute the normalizing constant $Z_4 = \int_{-\infty}^D \exp(\eta x)dx = \exp(\eta D)/\eta$. Next, we simply compute 
    \[\Pr(|x_4| \leq D-2/\eta) = \frac{\eta}{\exp(\eta D)} \int_{-D + 2/\eta}^{D - 2/\eta} \exp(\eta x) dx = \frac{1}{\exp(\eta D)} (e^{D\eta - 2} - e^{-D\eta + 2}) \geq \frac{1}{2e^2}, \eql{x4_lower}\]
    where the last inequality comes from the fact that $D \geq 3/\eta$. We let $f_{x_3} = \frac{1}{Z_3} \exp(\eta x_3- x_3^{d+1})$ and $f_{x_4} = \frac{1}{Z_4} \exp(\eta x_4)$ be pdf functions for $x_3$ and $x_4$ respectively. Applying Lemma \ref{Ratio Lemma} with $A = [-D + 2/\eta, D - 2/\eta]$ to these pdfs, we can conclude that
    \[\Pr(|x_3| \leq D - 2/\eta) \geq \Pr(|x_4| \leq D - 2/\eta). \eql{x4_bound}\]
    Combining Equations \eqref{eq:x1_bound} \eqref{eq:x2_bound}, \eqref{eq:x3_bound}, \eqref{eq:x4_lower}, and \eqref{eq:x4_bound} together, we get our desired result.
\end{proof}

\section{Proof of Fact \ref{moitra_fact}} \label{fact_proof}

\begin{proof}
Let $v \in \R^n$ be some vector satisfying $\|v\|_\infty = 1$, and let $x_v$ be the probability distribution for $x$ conditioned on $x$ lying in the span of $v$. If $\|x_v\|_\infty$ satisfies the tailbound for any choice of vector $v$, then $\|x\|_\infty$ will as well.

Defining $x_v = \alpha v$, we see that $\|x_v\|_\infty = |\alpha|$. Furthermore, we see that $\alpha$ has a probability density function of
\[p_\alpha(\alpha) = \frac{1}{Z_\alpha}\exp(-\sum\limits_{i \in [n]} (\alpha v_i)^{d} + \sum\limits_{k \in \K} \theta^*_k f_k(\alpha v)),\]
for appropriate value of $Z_\alpha$. We define a new variable $\beta$ which is sampled according to a distribution with pdf 
\[p_\beta(\beta) = \frac{1}{Z_\beta} \exp(-\beta^{d} + nB \beta^{d-1}),\]
for appropriate value of $Z_{\beta}$. For any $1 \leq s_2 < s_1$, we have that 
\begin{align*}
\frac{p_\alpha(s_1)}{p_\alpha(s_2)} &= 
\exp(\sum\limits_{k \in \K} \theta^*_k (f_k(s_1 v) - f_k(s_2 v)) - \sum\limits_{i \in [n]} ((s_1 v_i)^{d} - (s_2 v_i)^{d}))\\
&\geq \exp(\sum\limits_{k \in \K} \theta^*_k (s_1^{d-1} - s_2^{d-1}) - s_1^{d} + s_2^{d})\\
&\geq \exp(nB (s_1^{d-1} - s_2^{d-1}) - s_1^{d} + s_2^{d})\\
&= \frac{p_\beta(s_1)}{p_\beta(s_2)}.
\end{align*}
Thus, by applying Lemma \ref{Ratio Lemma}, we can conclude that for any $s > 1$, $\Pr(\alpha > s) \leq \Pr(\beta > s)$. Further, by taking the same argument applied to $- \alpha$, we can conclude that for any $s > 1$, $\Pr(|\alpha| > s) \leq 2\Pr(\beta > s)$. All that remains is to prove a tailbound for $\beta$. First, we see that
\[Z_\beta = \int_{-\infty}^\infty \exp(-\beta^{d} + nB \beta^{d-1}) d\beta \geq \int_1^{nB} \exp(-\beta^{d} + nB \beta^{d-1}) d \beta \geq nB. \]
Thus, for any $s > nB + 1$, we can see that 
\[\Pr(\beta > s) = \frac{1}{Z_\beta} \int_{s}^\infty \exp(-\beta^{d} + nB \beta^{d-1}) d\beta \leq \int_s^\infty \exp(-\beta^{d-1}) d \beta \leq \exp(-s^{d-1}).\]
The last inequality holds from substituting $u = \beta^{d-1}$ and seeing $\frac{du}{d\beta} = (d-1) \beta^{d-2} \geq 1$.

\end{proof}

\section{Curvature Concentration Bounds} \label{Curv_Conc_Sec}



This section is dedicated to proving concentration bounds on the curvature of the loss function and formalizing the argument in Section \ref{Curvature Analysis}. We start with the concentration result for $\delta \L_i(\Delta, \theta^*, x^{(M)})$.

\begin{lemma} \label{Curvature_Conc}
Let $x^{(1)}, \ldots x^{(M)}$ be independent samples drawn from $p_{\theta^*}$. Then, for any $\epsilon > 0$, 
\[\delta \L_i(\Delta, \theta^*, x^{(M)}) \geq \E[\delta \L_i(\Delta, \theta^*, x^{(M)})] - \frac{\epsilon B^2}{2} =  \L_i(\Delta, x^*) - \frac{\epsilon B^2}{2},\]
holds for every $\Delta = \hat \theta - \theta^*$ corresponding to valid parameter vector $\hat \theta$ with probability greater than $1 - \frac{3d^4 C_t^{4d - 4} n^{2d}}{M \epsilon^2}$. 
\end{lemma}

\begin{proof}
    Let $H_{k_1, k_2} = \E_{p_{\theta^*}} [\frac{\partial}{\partial x_i} f_{k_1}(x) \frac{\partial}{\partial x_i} f_{k_2}(x)]$ and let $\hat H_{k_1, k_2} = \hat \E[ \frac{\partial}{\partial x_i} f_{k_1}(x) \frac{\partial}{\partial x_i} f_{k_2}(x)]$, where $\hat \E$ denote the empirical average over our $M$ samples. We will show that the entries of these two matrices vary only slightly with high probability, and from that result show that the curvature is close to its expectation with high probability.

    Fix some $k_1, k_2 \in \K_i$, and notice that $\left |\frac{\partial}{\partial x_i} f_{k_1}(x) \frac{\partial}{\partial x_i} f_{k_2}(x) \right | \leq d^2 \|x\|_\infty^{2d-2}$ by construction. Thus, $\Var \left ( d^2 \|x\|_\infty^{2d-2} \right ) \leq d^4 \E[\|x\|_\infty^{4d - 4}]$. From Condition \ref{Bound_Assum}, we know for any $s \geq C_t^{4d-4}$ that
\[\Pr(\|x\|_\infty^{4d-4} \geq s) = \Pr(\|x\|_\infty \geq s^{1/(4d-4)}) \leq \exp(-k\sqrt{s}).\]
Thus,
\begin{align*}
    \E[\|x\|_\infty^{4d-4}] &\leq C_t^{4d-4} +  \int_{C_t^{4d-4}}^\infty \Pr(\|x\|_\infty^{4d-4} > s) ds\\
    &\leq C_t^{4d-4} +  \int_{C_t^{4d-4}}^\infty \exp(-k\sqrt{s}) ds\\
    &\leq C_t^{4d-4} +  \frac{2}{k^2} \exp(-k C_t^{2d-2}) (1 + k C_t^{2d-2})\\
    &\leq 3 C_t^{4d-4}. \eql{inf_var}
\end{align*}
We apply Chebyshev's Inequality to see that
\[\Pr(|H_{k_1, k_2} - \hat H_{k_1, k_2}| \geq \epsilon ) \leq \frac{\Var(\hat H_{k_1, k_2})}{\epsilon^2}  \leq \frac{ 3 d^4 C_t^{4d-4}}{M \epsilon^2}.\]
Notice that $|\K_i| \leq n^d$. Taking a union bound over all possible $k_1, k_2 \in \K_i$, we get that
\[\Pr(\|H - \hat H\|_\infty \geq \epsilon) \leq \frac{3d^4 n^{2d} C_t^{4d - 4}}{M \epsilon^2}.\]
Following the logic in equation \ref{eq:curvature}, and conditioned on $\|H - \hat H\|_\infty \geq \epsilon$, we see that for every $\Delta$
\begin{align*}
\delta \L_i(\Delta, \theta^*, x^{(M)}) & = \frac{\Delta_i^T \hat H \Delta_i}{2} \\
&= \frac{\Delta_i^T  H \Delta_i + \Delta_i^T (\hat H - H) \Delta_i}{2}\\
&= \delta \L_i(\Delta, x^*) + \frac{\Delta_i^T (\hat H - H) \Delta_i}{2}\\
&\geq \L_i(\Delta, x^*) - \frac{\epsilon\|\Delta_i\|_1^2}{2}\\
&\geq \L_i(\Delta, x^*) - \frac{\epsilon B^2}{2},
\end{align*}
where $\Delta_i$ denote the entries in $\Delta$ corresponding to some $k \in \K_i$. 
The last inequality follows from $\Delta = \hat \theta - \theta^*,$ where both satisfy a $B$ $l1$-bound. This conclude the proof.
    
\end{proof}

Next, we prove a probabilistic upper bound for $\|\nabla \L_i(\theta, x^{(M)})\|_\infty$.
\begin{lemma} \label{Curv_Inf}
Let $x^{(1)}, \ldots x^{(M)}$ be independent samples drawn from $p_{\theta^*}$. Then, for any $\epsilon >0$, $\Pr(\|\nabla \L_i(\theta^*, x^{(M)})\|_\infty \geq \epsilon) \leq \frac{12 d^2 B^2 C_t^{4d-4} n^d}{M \epsilon^2}$.
\end{lemma}

\begin{proof}
For any collection of samples $x_1, \ldots x_M$, we see that 
\[\|\nabla_\theta \L_i(\theta^*, x^{(M)})\|_\infty = \max_{k \in \K} \frac{1}{M} \sum\limits_{j=1}^M \( \frac{\partial}{\partial x_i^2} f_k(x^{(j)}) + 2 \frac{\partial}{\partial x_i} f_k(x^{(j)}) \sum\limits_{k' \in \K_i} \theta_{k'} \frac{\partial}{\partial x_i} f_{k'}(x^{(j)})\).\]
We let $\alpha_k$ denote the expression inside of the max operator above for each $k$.
We know that $\E_{p_{\theta^*}}[\alpha_k] = 0,$ since $\theta^*$ is the expected minimizer of $\L_i$.
Further, notice that \\
$\alpha_k \leq \frac{1}{M} \sum\limits_{j=1}^M 2dB \|x^{(j)}\|^{2d-2}_\infty$, implying that 
\[\Var(\alpha_k) = \E[\alpha_k^2] \leq 4d^2 B^2 \E[\|x\|^{4d-4}_\infty]/M \leq 12 d^2 B^2 C_t^{4d-4}/M\]
by Equation \eqref{eq:inf_var}. 
We can now apply Chebyshev's Inequality to see that
\[\Pr(\frac{1}{M} \sum\limits_{j=1}^M \alpha_i \geq \epsilon) \leq \frac{\Var(\frac{1}{M} \sum\limits_{j=1}^M \alpha_i)}{\epsilon^2} \leq \frac{12d^2B^2 C_t^{4d-4}}{M \epsilon^2}.\]
Observing that $|\K_i| \leq n^d$ and applying a union bound to the above will prove our result.
\end{proof}

Finally, we combine the previous two Lemmas with Equation \eqref{eq:0-bound} in the following Lemma to complete the argument outlined in Section \ref{Curvature Analysis}.

\begin{lemma}\label{main_lemma}
Take target factors $\hat \K \subseteq \K$ which satisfy $\E_{x \sim p_t} E_i(x,\Delta)^2 \geq C_p \Delta_k^2$ for every $k \in \hat \K$. Then, for any $\rho > 0$, the empirical minimizer $\hat \theta$ of $\L_i(\theta, x^{(M)})$ satisfies
$(\theta^*_k - \hat \theta_k)^2 \leq \epsilon $ for every $k \in \hat \K$ with probability 
\[1 -  \frac{1}{\rho n C_t}, \]
as long as 
\[M \geq \rho \frac{3120 d^4 B^4 C_t^{4d} n^{2d+1}}{\epsilon^2  C_p^2}.\]
\end{lemma}

\begin{proof}
Recall that $\Delta = \hat \theta - \theta^*$ and let $m_\Delta = \max_{k \in \hat \K} \Delta_k^2$.
Continuing from Equation \eqref{eq:0-bound}, we observe that $\hat \theta$ being the empirical minimizer implies 
\[\delta \L_i(\Delta, \theta^*, x^{(M)}) \leq 2B \|\nabla_\theta \L_i(\theta^*, x^{(M)})\|_\infty.\]
Combining Equations \eqref{eq:curvature} and \eqref{eq:tailbound}, we know that 
\[\delta \L_i(\Delta, \theta^*) \geq \frac{1}{4} \E_{x \sim p_{t}} E_i[x, \Delta]^2 \geq \frac{1}{4} C_p m_\Delta.\]
By plugging $\epsilon' = \frac{C_p \epsilon}{4 B^2}$ into Lemma \ref{Curvature_Conc}, we see that if $m_\Delta > \epsilon$, then
\[\delta \L_i(\Delta, \theta^*, x^{(M)}) \geq \L_i(\Delta, x^*) - \frac{\epsilon' B^2}{2} \geq \frac{1}{4} C_p m_\Delta - \frac{1}{8} C_p \epsilon > \frac{1}{8}C_p \epsilon
, \eql{curv1}\]
with probability at least $1 - \frac{48 d^4 B^4 C_t^{4d - 4} n^{2d}}{M C_p^2 \epsilon^2}$. Further, by Lemma \ref{Curv_Inf}, we know that
\[
\|\nabla_\theta \L_i(\theta^*, x^{(M)})\|_\infty \leq \frac{C_p \epsilon}{16 B}, \eql{curv2}
\]
with probability $1 - \frac{3072 d^2 B^4 C_t^{4d-4} n^d}{M C_p^2 \epsilon^2}$. Since Equations \eqref{eq:curv1} and \eqref{eq:curv2} cannot both hold without contradicting $\hat \theta$ being the empirical minimizer, we can conclude via the union bound that
\begin{align*}
\Pr(m_\Delta > \epsilon) 
&\leq \frac{48 d^4 B^4 C_t^{4d - 4} n^{2d}}{M C_p^2 \epsilon^2} + \frac{3072 d^2 B^4 C_t^{4d-4} n^d}{M C_p^2 \epsilon^2}\\
&\leq \frac{3120 d^4 B^4 C_t^{4d - 4} n^{2d}}{M C_p^2 \epsilon^2}\\
&\leq \frac{1}{\rho n C_t}
\end{align*}

\end{proof}

\section{Near-Optimality Recovery Guarantee} \label{approx_sec}

In practice, it may be difficult to compute the exact minimizer of the score-matching loss. However, our analysis of the loss's curvature allows us to prove recovery guarantees for \textit{approximate} minimizers of the score-matching loss as well. Specifically, given samples $x^{(M)}$, we say a solution $\tilde \theta$ is a $\eta$-minimizer of $\L_i(\theta, x^{(M)})$ if for some $\eta > 0$,
\[\L_i(\tilde \theta, x^{(M)}) - \L_i(\hat \theta, x^{(M)}) \leq \eta,\]
where $\hat \theta$ is the true minimizer.
This section will be dedicated to proving the following result, which extends our main result Theorem \ref{fam_struc_thm} for a $\eta$-minimizer $\tilde \theta$. A similar process can be followed for our other theorems as well.

\begin{theorem}[Learning from Near-Optimality] \label{near_opt_thm}
Fix some exponential family $p_{\theta^*}$ with base measure $h(x) = 1$. For a fixed index $i \in [n]$, let $\hat \K$ be the maximal factors with $i$ as a neighbor in the family factor graph $G$, meaning $\hat \K = \{k \in \M_{\rm fac}(G)| i \in \partial k\}.$ Further, for $M$ independent samples $x_1, \ldots x_M \sim p_{\theta^*}$, let $\tilde \theta$ be a $\eta$-minimizer of $\L_i(\theta, x^{(M)})$ subject to $\sum\limits_{k \in \K_j} |\theta_k| \leq B$ for every index $j$.

There exists some $M^*$ and $\eta^*$ both of the order $(d B C_t^d)^{O(d^2 w)}$ such that for every $\rho \geq 1$ and $\epsilon \leq 1$, we have that $(\theta^*_k - \hat \theta_k)^2 \leq \epsilon$ for every $k \in \hat \K$ with probability greater than $\frac{1}{\rho n C_t}$ when $M \geq \rho \frac{n^{d+1} M^*}{\epsilon^2}$ and $\eta \leq \epsilon \eta^*$.
\end{theorem}
To proceed, we will first need to show the following adapted version of Lemma \ref{main_lemma}. With this, Theorem \ref{near_opt_thm} follows directly from applying the Lemma to Equation \eqref{eq:struct_curv_bound}, just as the main Theorem does.

\begin{lemma}\label{approx_lemma}
Take target factors $\hat \K \subseteq \K$ which satisfy $\E_{x \sim p_t} E_i(x,\Delta)^2 \geq C_p \Delta_k^2$ for every $k \in \hat \K$ and fix some $\epsilon > 0$. 
Let $\tilde \theta$ be an $\eta$-minimizer of $\L_i(\theta, x^{(M)})$, for some $\eta < \frac{1}{8}C_p \epsilon$.
Then, for any $\rho > 0$, $\tilde \theta$ satisfies
$(\theta^*_k - \tilde \theta_k)^2 \leq \epsilon $ for every $k \in \hat \K$ with probability 
\[1 -  \frac{1}{\rho n C_t}, \]
as long as 
\[M \geq \rho \frac{3264 d^4 B^4 C_t^{4d} n^{2d+1}}{\epsilon^2  C_p^2}.\]
\end{lemma}

\begin{proof}
Set $\Delta = \tilde \theta - \theta^*$ and let $m_\Delta = \max_{k \in \hat \K} \Delta_k^2$. We first observe that $\delta \L_i(\Delta, \tilde \theta, x^{(M)}) - \delta \L_i(\Delta, \theta^*, x^{(M)}) \leq \eta$ by the definition of $\eta$-minimizer.
Adapting Equation \eqref{eq:0-bound}, we observe that 
\[\delta \L_i(\Delta, \theta^*, x^{(M)}) \leq 2B \|\nabla_\theta \L_i(\theta^*, x^{(M)})\|_\infty + \eta < 2B \|\nabla_\theta \L_i(\theta^*, x^{(M)})\|_\infty + \frac{1}{8} C_p \epsilon.\]
Combining Equations \eqref{eq:curvature} and \eqref{eq:tailbound}, we know that 
\[\delta \L_i(\Delta, \theta^*) \geq \frac{1}{4} \E_{x \sim p_{t}} E_i[x, \Delta]^2 \geq \frac{1}{4} C_p m_\Delta.\]
By plugging $\epsilon' = \frac{C_p \epsilon}{8 B^2}$ into Lemma \ref{Curvature_Conc}, we see that if $m_\Delta > \epsilon$, then
\[\delta \L_i(\Delta, \theta^*, x^{(M)}) \geq \L_i(\Delta, x^*) - \frac{\epsilon' B^2}{2} \geq \frac{1}{4} C_p m_\Delta - \frac{1}{16} C_p \epsilon > \frac{3}{16}C_p \epsilon
, \eql{curv3}\]
with probability at least $1 - \frac{192 d^4 B^4 C_t^{4d - 4} n^{2d}}{M C_p^2 \epsilon^2}$. Further, by Lemma \ref{Curv_Inf}, we know that
\[
\|\nabla_\theta \L_i(\theta^*, x^{(M)})\|_\infty \leq \frac{C_p \epsilon}{16 B}, \eql{curv4}
\]
with probability $1 - \frac{3072 d^2 B^4 C_t^{4d-4} n^d}{M C_p^2 \epsilon^2}$. Since Equations \eqref{eq:curv3} and \eqref{eq:curv4} cannot both hold without contradicting $\tilde \theta$ being an $\eta$-minimizer, we can conclude via the union bound that
\begin{align*}
\Pr(m_\Delta > \epsilon) 
&\leq \frac{192 d^4 B^4 C_t^{4d - 4} n^{2d}}{M C_p^2 \epsilon^2} + \frac{3072 d^2 B^4 C_t^{4d-4} n^d}{M C_p^2 \epsilon^2}\\
&\leq \frac{3264 d^4 B^4 C_t^{4d - 4} n^{2d}}{M C_p^2 \epsilon^2}\\
&\leq \frac{1}{\rho n C_t}
\end{align*}

\end{proof}

\end{document}